\documentclass[letterpaper]{article} 
\usepackage{aaai25}  
\usepackage{times}  
\usepackage{helvet}  
\usepackage{courier}  
\usepackage[hyphens]{url}  
\usepackage{graphicx} 
\usepackage{natbib}  
\usepackage{caption} 
\usepackage{algorithm}
\usepackage{algorithmic}

\usepackage{amsmath}
\usepackage{amssymb}
\usepackage{mathtools}
\usepackage{amsthm}
\usepackage{multirow}
\usepackage{enumitem}
\usepackage{blindtext}
\usepackage{color}
\usepackage{mathtools}
\usepackage{booktabs}

\usepackage{amsmath,amsfonts,bm}

\def\eqref#1{equation~\ref{#1}}

\def\1{\bm{1}}

\def\vr{{\bm{r}}}

\def\vx{{\bm{x}}}

\DeclareMathAlphabet{\mathsfit}{\encodingdefault}{\sfdefault}{m}{sl}
\SetMathAlphabet{\mathsfit}{bold}{\encodingdefault}{\sfdefault}{bx}{n}

\newcommand{\R}{\mathbb{R}}

\definecolor{C0}{RGB}{31,119,180}
\definecolor{C1}{RGB}{255,127,14}
\definecolor{C2}{RGB}{44,160,44}
\definecolor{C3}{RGB}{214,39,40}
\definecolor{C4}{RGB}{148,103,189}
\definecolor{C5}{RGB}{140,86,75}
\definecolor{C6}{RGB}{227,119,194}
\definecolor{C7}{RGB}{127,127,127}
\definecolor{C8}{RGB}{188,189,34}
\definecolor{C9}{RGB}{23,190,207}

\renewcommand{\eqref}[1]{(\ref{#1})}
\theoremstyle{plain}
\newtheorem{theorem}{Theorem}
\newtheorem*{theorem*}{Theorem}

\newtheorem{lemma}[theorem]{Lemma}

\theoremstyle{definition}
\newtheorem{definition}[theorem]{Definition}

\theoremstyle{remark}

\usepackage{newfloat}
\usepackage{listings}
\DeclareCaptionStyle{ruled}{labelfont=normalfont,labelsep=colon,strut=off} 
\floatstyle{ruled}
\newfloat{listing}{tb}{lst}{}
\floatname{listing}{Listing}
\title{Number Theoretic Accelerated Learning of Physics-Informed Neural Networks}
\author {
    Takashi Matsubara\textsuperscript{\rm 1},
    Takaharu Yaguchi\textsuperscript{\rm 2}
}
\affiliations {
    \textsuperscript{\rm 1}Hokkaido University\\
    \textsuperscript{\rm 2}Kobe University\\
    matsubara@ist.hokudai.ac.jp, yaguchi@pearl.kobe-u.ac.jp
}

\begin{document}

\maketitle

\begin{abstract}
Physics-informed neural networks solve partial differential equations by training neural networks.
Since this method approximates infinite-dimensional PDE solutions with finite collocation points, minimizing discretization errors by selecting suitable points is essential for accelerating the learning process.
Inspired by number theoretic methods for numerical analysis, we introduce good lattice training and periodization tricks, which ensure the conditions required by the theory.
Our experiments demonstrate that GLT requires 2--7 times fewer collocation points, resulting in lower computational cost, while achieving competitive performance compared to typical sampling methods.
\end{abstract}

%

\section{Introduction}
Many real-world phenomena can be modeled as partial differential equations (PDEs), and solving PDEs has been a central topic in computational science.
The applications include, but are not limited to, weather forecasting, vehicle design~\cite{Hirsch2006}, economic analysis~\cite{Achdou2014}, and computer vision~\cite{Logan2015}.
A PDE is expressed as $\mathcal{N}[u]=0$, where $\mathcal{N}$ is a (possibly nonlinear) differential operator, and $u:\Omega\rightarrow\R$ is an unknown function on the domain $\Omega\subset\R^s$.
For most PDEs that appear in physical simulations, the well-posedness (the uniqueness of the solution $u$ and the continuous dependence on the initial and boundary conditions) has been well-studied and is typically guaranteed under certain conditions.
To solve PDEs, various computational techniques have been explored, including finite difference methods, finite volume methods, and spectral methods~\cite{Furihata2010,Morton2005,Thomas1995}.
However, the development of computer architecture has become slower, leading to a growing need for computationally efficient alternatives.
A promising approach is physics-informed neural networks (PINNs)~\cite{Raissi2019}, which train a neural network by minimizing the physics-informed loss~\cite{Wang2022b,Wang2023}.
This is typically the squared error of the neural network's output $\tilde{u}$ from the PDE $\mathcal{N}[u]=0$ averaged over a finite set of collocation points $\bm{x}_j$, $\frac{1}{N}\sum_{j=0}^{N-1}\|\mathcal{N}[\tilde{u}](\bm{x}_j)\|^2$, encouraging the output $\tilde{u}$ to satisfy the equation $\mathcal{N}[\tilde{u}](\bm{x}_j)=0$.

\begin{figure}[t]
  \centering
  \footnotesize
  \tabcolsep=0mm
  \begin{tabular}{ccc}
    \hspace*{3mm}\includegraphics[scale=0.8]{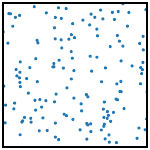}\hspace*{3mm} &
    \hspace*{3mm}\includegraphics[scale=0.8]{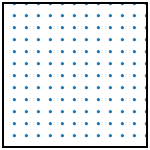} \hspace*{3mm} &
    \hspace*{3mm}\includegraphics[scale=0.8]{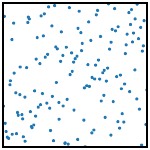}\hspace*{3mm}   \\[-0.7mm]
    uniformly random                                                             &
    uniformly spaced                                                             &
    LHS                                                                            \\[1.mm]
    \hspace*{3mm}\includegraphics[scale=0.8]{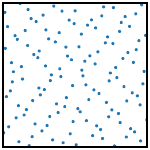}\hspace*{3mm} &
    \hspace*{3mm}\includegraphics[scale=0.8]{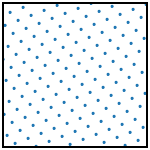}\hspace*{3mm} &
    \hspace*{3mm}\includegraphics[scale=0.8]{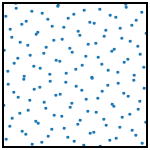}\hspace*{3mm}  \\[-0.7mm]
    Sobol sequence                                                               &
    proposed GLT                                                                 &
    \hspace*{-3mm}proposed GLT (folded)\hspace*{-3mm}                              \\
  \end{tabular}
  \caption{Examples of sampled collocation points.
    128 points for the Sobol sequence, and 144 points for others.}
  \label{fig:collocation}
\end{figure}

However, the solutions $u$ to PDEs are inherently infinite-dimensional, and any distance involving the output $\tilde{u}$ or the solution $u$ needs to be defined by an integral over the domain $\Omega$.
In this regard, the physics-informed loss serves as a finite approximation to the squared 2-norm $|\mathcal{N}[\tilde{u}]|^2_2=\int_{\bm{x}\in\Omega} \|\mathcal{N}[\tilde{u}](\bm{x})\|^2\mathrm{d}\bm{x}$ on the function space $L^2(\Omega)$ for $\mathcal{N}[u]\in L^2(\Omega)$, and hence the discretization errors should affect the training efficiency.
A smaller number $N$ of collocation points leads to a less accurate approximation and inferior performance, while a larger number $N$ increases the computational cost~\cite{Bihlo2022,Sharma2022}.
Despite the importance of selecting appropriate collocation points, insufficient emphasis has been placed on this aspect.
\citet{Raissi2019}, \citet{Zeng2023}, and many other studies employed Latin hypercube sampling (LHS) to determine the collocation points.
Alternative approaches include uniformly random sampling (i.e., the Monte Carlo method)~\cite{Jin2021,Krishnapriyan2022} and uniformly spaced sampling~\cite{Wang2021c,Wang2022a}.
These methods are exemplified in Fig.~\ref{fig:collocation}.

In the field of numerical analysis, the relationship between integral approximation and collocation points has been extensively investigated.
Accordingly, some studies have used quasi-Monte Carlo methods, specifically the Sobol sequence, which approximate integrals more accurately than the Monte Carlo method~\cite{Lye2020,Longo2021,Mishra2021}.
For further improvement, this paper proposes \emph{good lattice training (GLT)} for PINNs and their variants, such as the competitive PINN (CPINN)~\cite{Zeng2023} and physics-informed neural operator~\cite{Li2021e,Rosofsky2023}.
The GLT is inspired by number theoretic methods for numerical analysis, providing an optimal set of collocation points depending on the initial and boundary conditions, as shown in Fig.~\ref{fig:collocation}.
Our experiments demonstrate that the proposed GLT requires far fewer collocation points than comparison methods while achieving similar errors, significantly reducing computational cost.
The contribution and significance of the proposed GLT are threefold.

\vspace*{0.5mm}\noindent\textbf{Computationally Efficient:}
The proposed GLT offers an optimal set of collocation points to compute a loss function that can be regarded as a finite approximation to an integral over the domain, such as the physics-informed loss, if the activation functions of the neural networks are smooth enough.
It requires significantly fewer collocation points to achieve accuracy of solutions and system identifications comparable to other methods, or can achieve lower errors with the same computational budget.

\vspace*{0.5mm}\noindent\textbf{Applicable to PINNs Variants:}
As the proposed GLT changes only the collocation points, it can be applied to various variants of PINNs without modifying the learning algorithm or objective function.
In this study, we investigate a specific variant, the CPINNs~\cite{Zeng2023}, and demonstrate that CPINNs using the proposed GLT achieve superior convergence speed with significantly fewer collocation points than CPINNs using LHS.

\vspace*{0.5mm}\noindent\textbf{Theoretically Solid:}
Number theory provides a theoretical basis for the efficacy of the proposed GLT.
Existing methods based on quasi-Monte Carlo methods are inferior to the proposed GLT in theoretical performance, or at least require the prior knowledge about the smoothness $\alpha$ of the solution $u$ and careful adjustments of hyperparameters~\cite{Longo2021}.
On the other hand, the proposed GLT is free from these prior knowledge or adjustments and achieves better performances depending on the smoothness $\alpha$ and the neural network, which is a significant advantage.

\section{Related Work}
Neural networks are a powerful tool for processing information and have achieved significant success in various fields~\cite{He2015a,Vaswani2017}, including black-box system identification, that is, to learn the dynamics of physical phenomena from data and predict their future behaviors~\cite{Chen2018e,Chen1990,Wang1998}.
By integrating knowledge from analytical mechanics, neural networks can learn dynamics that adheres to physical laws and even uncover these laws from data~\cite{Finzi2020,Greydanus2019,Matsubara2023ICLR}.

Neural networks have also gained attention as computational tools for solving differential equations, particularly PDEs~\cite{Dissanayake1994,Lagaris1998}.
Recently, \citet{Raissi2019} introduced an elegant refinement to this approach and named it PINNs.
The key concept behind PINNs is the physics-informed loss~\cite{Wang2022b,Wang2023}.
This loss function evaluates the extent to which the output $\tilde{u}$ of the neural network satisfies a given PDE $\mathcal{N}[u]=0$ and its associated initial and boundary conditions $\mathcal B[u]=0$.
The physics-informed loss can be integrated into other models like DeepONet~\cite{Lu2021a,Wang2023} or used for white-box system identifications (that is, adjusting the parameters of known PDEs so that their solutions fit observations).

PINNs are applied to various PDEs~\cite{Bihlo2022,Jin2021,Mao2020}, with significant efforts in improving learning algorithms and objective functions~\cite{Hao2023,Heldmann2023,Lu2022,Pokkunuru2023,Sharma2022,Zeng2023}.
Objective functions are generally based on PDEs evaluated at a finite set of collocation points rather than data.
\citet{Bihlo2022,Sharma2022} have shown a trade-off between the number of collocation points (and hence computational cost) and the accuracy of the solution.
Thus, selecting collocation points that efficiently cover the entire domain $\Omega$ is essential for achieving better results.
Some studies have employed quasi-Monte Carlo methods, specifically the Sobol sequence, to determine the collocation points~\cite{Lye2020,Mishra2021}, but their effectiveness depends on knowledge of the solution's smoothness $\alpha$ and hyperparameter adjustments~\cite{Longo2021}.

\section{Method}
\paragraph{Theoretical Error Estimate of PINNs}
For simplicity, we consider PDEs defined on an $s$-dimensional unit cube $[0, 1]^s$.
PINNs employ a PDE that describes the target physical phenomena as loss functions.
Specifically, first, an appropriate set of collocation points $L^*=\{ \bm{x}_j \mid j=0,\ldots,N-1 \}$ is determined, and then the sum of the residuals of the PDE at these points
\begin{equation}\label{eq:dloss}
  \textstyle \frac{1}{N}\sum_{j=0}^{N-1} \mathcal{P}[\tilde u](\bm{x}_j)=\frac{1}{N}\sum_{\vx_j\in L^*} \mathcal{P}[\tilde u](\bm{x}_j),
\end{equation}
is minimized as a loss function, where $\mathcal{P}$ is a differential operator.
The physics-informed loss satisfies $\mathcal{P}[\tilde u](\bm{x})=\| \mathcal{N}[\tilde{u}](\bm{x}) \|^2$.
Then, the neural network's output $\tilde{u}$ becomes an approximate solution of the PDE.
However, for $\tilde{u}$ to be the exact solution, the loss function should be 0 for all $\bm{x}\in\Omega$, not just at collocation points.
Therefore, the following integral must be minimized as the loss function:
\begin{equation}\label{eq:closs}
  \textstyle \int_{[0, 1]^s} \mathcal{P}[\tilde{u}](\bm{x}) \mathrm{d}\bm{x}.
\end{equation}
In other words, the practical minimization of \eqref{eq:dloss} essentially minimizes the approximation of \eqref{eq:closs} with the expectation that {\eqref{eq:closs}} will be small enough, and hence $\tilde{u}$ becomes an accurate approximation to the exact solution.

More precisely, we show the following theorem, which is an improvement of an existing error analysis \cite{Mishra2023estimates} in the sense that the approximation error bound of neural networks is considered.
\begin{theorem}\label{thm:3.1}
  Suppose that the class of neural networks used for PINNs includes an $\varepsilon_1$-approximator $\tilde{u}_\mathrm{opt}$ to the exact solution $u^*$ to the PDE $\mathcal{N}[u]=0$:
  $\|u^* - \tilde{u}_\mathrm{opt} \| \leq \varepsilon_1$,
  and that \eqref{eq:dloss} is an $\varepsilon_2$-approximation of $\eqref{eq:closs}$ for the approximated solution $\tilde{u}$ and also for $\tilde{u}_\mathrm{opt}$:
  $| \textstyle \int_{[0, 1]^s} \mathcal{P}[{u}](\bm{x}) \mathrm{d}\bm{x} - \frac{1}{N}\sum_{\vx_j\in L^*} \mathcal{P}[u](\bm{x}_j)| \leq \varepsilon_2$
  for $u=\tilde{u}$ and $u=\tilde{u}_\mathrm{opt}$.
  Suppose also that there exist $c_\mathrm{p} > 0$ and $c_\mathrm{L} > 0$ such that
  $
    \frac{1}{c_\mathrm{p}} \|u - v\|
    \leq \| \mathcal{N}[u] - \mathcal{N}[v] \|
    \leq c_\mathrm{L} \| u - v \|.
  $
  Then,
  \begin{equation*}
    \textstyle
    \|u^* - \tilde{u}\|
    \leq (1 + c_\mathrm{p} c_\mathrm{L}) \varepsilon_1
    + c_\mathrm{p} \sqrt{\frac{1}{N}\sum_{\vx_j\in L^*} \mathcal{P}[\tilde u](\bm{x}_j) + \varepsilon_2}.
  \end{equation*}
\end{theorem}
For a proof, see Appendix ``Theoretical Background.''
$\varepsilon_1$ is determined by the architecture of the network and the function space to which the solution belongs.
For example, approximation rates of neural networks in Sobolev spaces are given in \citet{GUHRING2021107}.
This explains why increasing the number of collocation points beyond a certain point does not further reduce the error.
$\varepsilon_2$ depends on the accuracy of the approximation of the integral.
In this paper, we investigate a training method that easily gives small $\varepsilon_2$.

One standard strategy often used in practice is to set $\bm{x}_j$'s to be uniformly distributed random numbers, which can be interpreted as the Monte Carlo approximation of the integral \eqref{eq:closs}.
As is widely known, the Monte Carlo method can approximate the integral within an error of $O(1/N^{\frac{1}{2}})$ independently from the number $s$ of dimensions~\cite{Sloan1994-cl}.
However, most PDEs for physical simulations are two to four-dimensional, incorporating a three-dimensional space and a one-dimensional time.
Hence, in this paper, we propose a sampling strategy specialized for low-dimensional cases, inspired by number-theoretic numerical analysis.

Note that some variants, such as CPINN~\cite{Zeng2023}, have proposed alternative objective functions.
We hereafter denote any variant of physics-informed loss by \eqref{eq:dloss}, without loss of generality, as long as it can be regarded as a finite approximation to an integral over a domain, \eqref{eq:closs}.

\paragraph{Good Lattice Training}
In this section, we propose the \emph{good lattice training (GLT)}, in which a number theoretic numerical analysis is used to accelerate the training of PINNs~\cite{Niederreiter1992-qb, Sloan1994-cl, Zaremba1972-gj}.
In the following, we use some tools from this theory.

While our target is a PDE on the unit cube $[0, 1]^s$, we now treat the loss function $\mathcal{P}[\tilde u]$ as periodic on $\R^s$ with a period of 1.
Then, we define a lattice.
\begin{definition}[\citet{Sloan1994-cl}]
  A lattice $L$ in $\mathbb{R}^s$ is defined as a finite set of points in $\mathbb{R}^s$ that is closed under addition and subtraction.
\end{definition}
Given a lattice $L$, the set of collocation points is defined as $L^*=\{\bm{x}_j\mid j=0,\dots,N-1\}\coloneqq\{\mbox{the decimal part of }\bm{x} \mid \bm{x} \in L \}\in[0,1]^s$.
Considering that the loss function to be minimized is \eqref{eq:closs}, it is desirable to determine the lattice $L$ (and hence the set of collocation points $\bm{x}_j$'s) so that the difference $| \eqref{eq:closs} - \eqref{eq:dloss}|$ of the two functions is minimized.

Suppose that $\varepsilon(\bm{x}) \coloneqq \mathcal{P}[\tilde u](\bm{x})$ is smooth enough, admitting the Fourier series expansion:
\begin{equation*}
  \textstyle \varepsilon(\bm{x}) \coloneqq \mathcal{P}[\tilde u](\bm{x})= \sum_{\bm{h}} \hat{\varepsilon}(\bm{h}) \exp(2 \pi \mathrm{i} \bm{h} \cdot \bm{x}),
\end{equation*}
where $\mathrm{i}$ denotes the imaginary unit and $\bm{h}=(h_1,h_2,\ldots,h_s)$ $\in \mathbb{Z}^s$.
Substituting this into \eqref{eq:dloss} yields
\begin{equation}
  \textstyle | \eqref{eq:closs} - \eqref{eq:dloss}| = \left|\frac{1}{N} \sum_{j=0}^{N-1}    \sum_{\bm{h} \in \mathbb{Z}^s, \bm{h} \neq 0} \hat{\varepsilon}(\bm{h}) \exp(2 \pi \mathrm{i}  \bm{h} \cdot \bm{x}_j)\right|,
  \label{eq:interror0}
\end{equation}
because the Fourier mode of $\bm{h}=0$ is equal to the integral $\int_{[0, 1]^s} \varepsilon(\bm{x}) \mathrm{d}\bm{x}$.
Before optimizing \eqref{eq:interror0}, the dual lattice of lattice $L$ and an insightful lemma are introduced as follows.
\begin{definition}[\citet{Zaremba1972-gj}]
  A dual lattice $L^\top$ of a lattice $L$ is defined as $L^\top \coloneqq \{ \bm{h} \in \mathbb{R}^s \mid \bm{h} \cdot \bm{x} \in \mathbb{Z}, \ \forall \bm{x} \in L \}$.
\end{definition}
\begin{lemma}[\citet{Zaremba1972-gj}]\label{lemma:lemma}
  For $\bm{h} \in \mathbb{Z}^s$, it holds that
  \begin{equation*}
    \textstyle \frac{1}{N} \sum_{j=0}^{N-1}
    \exp( 2 \pi \mathrm{i} \bm{h} \cdot \bm{x}_j)
    =\begin{cases}
      1 & (\bm{h} \in L^\top)   \\
      0 & (\mathrm{otherwise}.)
    \end{cases}
  \end{equation*}
\end{lemma}
Lemma~\ref{lemma:lemma} follows directly from the properties of Fourier series.
Based on this lemma, we restrict the lattice point $L$ to the form $\{\bm{x} \mid \bm{x}=\frac{j}{N} \bm{z}\mbox{ for } j\in\mathbb{Z}\}$ with a fixed integer vector $\bm{z}$; the set $L^*$ of collocation points is $\{\mbox{the decimal part of\ }\frac{j}{N} \bm{z}\mid j=0,\ldots,N-1\}$.
Then, instead of searching $\bm{x}_j$'s, a vector $\bm{z}$ is searched.
By restricting to this form, $\bm{x}_j$'s can be obtained automatically from a given $\bm{z}$, and hence the optimal collocation points $\bm{x}_j$'s do not need to be stored as a table of numbers, making a significant advantage in implementation.
Another advantage is theoretical; the optimization problem of the collocation points can be reformulated in a number theoretic way. In fact, for $L$ as shown above, it is confirmed that $L^\top = \{ \bm{h} \mid \bm{h} \cdot \bm{z} \equiv 0 \pmod N \}$. If $\bm{h} \cdot \bm{z} \equiv 0 \pmod N$ then there exists an $m \in \mathbb{Z}$ such that $\bm{h} \cdot \bm{z} = m N$ and hence $\frac{j}{N}\bm{h} \cdot \bm{z} = m j \in \mathbb{Z}$. Conversely, if $\bm{h} \cdot \bm{z} \not\equiv 0 \pmod N$, clearly $\frac{1}{N}\bm{h} \cdot \bm{z} \notin \mathbb{Z}$.

From the above lemma,
\begin{equation}\label{eq:interror}
  \textstyle \eqref{eq:interror0}\le \sum_{\bm{h} \in \mathbb{Z}^s, \bm{h} \neq 0, \bm{h} \cdot \bm{z} \equiv 0 \pmod N} | \hat{\varepsilon}(\bm{h}) |,
\end{equation}
and hence the collocation points $\bm{x}_j$'s should be determined so that \eqref{eq:interror} becomes small.
This problem is a number theoretic problem in the sense that it is a minimization problem of finding an integer vector $\bm{h}$ subject to the condition $\bm{h} \cdot \bm{z}  \equiv 0 \pmod N$.
This problem has been considered in the field of number theoretic numerical analysis. In particular, optimal solutions have been investigated for integrands in the Korobov spaces, which are spaces of functions that satisfy a certain smoothness condition.
\begin{definition}[\citet{Zaremba1972-gj}]\label{definition:Korobov}
  The function space that is defined as $E_\alpha = \{ f:[0, 1]^s \to \R \mid {}^\exists c, |\hat{f}(\bm{h})| \leq \frac{c}{(\bar{h}_1 \bar{h}_2 \cdots \bar{h}_s)^\alpha}\}$ is called the Korobov space, where $\hat{f}(\bm{h})$ is the Fourier coefficients of $f$ and $\bar{k} = \max(1, |k|)$ for $k \in \mathbb{R}$.
\end{definition}
It is known that if $\alpha$ is an integer, for a function $f$ to be in $E_\alpha$, it is sufficient that $f$ has continuous partial derivatives $\frac{\partial^{q_1+ q_2 + \cdots + q_s}}{\partial_1^{q_1} \cdot \partial_2^{q_2} \cdots \partial_s^{q_s}}f, 0 \leq q_k \leq \alpha\ (k=1,\ldots,s)$. For example, if a function $f(x, y):\mathbb{R}^2 \to \mathbb{R}$ has continuous $f_x, f_y, f_{xy}$, then $f \in E_1$.
Hence, if $\mathcal{P}[\tilde{u}]$ and the neural network belong to Koborov space,
\begin{equation}\label{eq:palpha}
  \textstyle \eqref{eq:interror}\le \sum_{\bm{h} \in \mathbb{Z}^s, \bm{h} \neq 0, \bm{h} \cdot \bm{z} \equiv 0 \pmod N} \frac{c}{(\bar{h}_1 \bar{h}_2 \cdots \bar{h}_s)^\alpha}.
\end{equation}
Here, we introduce a theorem in the field of number theoretic numerical analysis:
\begin{theorem}[\citet{Sloan1994-cl}]\label{theorem:support}
  For integers $N \geq 2$ and $s \geq 2$, there exists a $\bm{z} \in \mathbb{Z}^s$ such that
  \begin{equation*}
    \textstyle P_\alpha(\bm{z}, N) \leq \frac{(2 \log N)^{\alpha s}}{N^\alpha} + O\left(\frac{(\log N)^{\alpha s -1}}{N^\alpha}\right).
  \end{equation*}
  for $P_\alpha(\bm{z}, N)=\frac{1}{(\bar{h}_1 \bar{h}_2 \cdots \bar{h}_s)^\alpha}$.
\end{theorem}

The main result of this paper is the following.
\begin{theorem}\label{theorem:goodlattice}
  Suppose that the activation function of $\tilde{u}$ and hence $\tilde{u}$ itself are sufficiently smooth so that there exists an $\alpha > 0$ such that $\mathcal{P}[\tilde{u}] \in E_\alpha$. Then, for given integers $N \geq 2$ and $s \geq 2$, there exists an integer vector $\bm{z} \in \mathbb{Z}^s$ such that $L^*=  \{ \mbox{the decimal part of\ } \frac{j}{N} \bm{z}  \mid j=0,\ldots,N-1\}$ is a ``good lattice'' in the sense that
  \begin{equation}\label{eq:goodlattice}
    \textstyle \left|\int_{[0, 1]^s} \!\!\! \mathcal{P}[\tilde u](\bm{x}) \mathrm{d}\bm{x}
    \!-\!
    \frac{1}{N}\!\!\sum_{\bm{x}_j \in L^*} \!\! \mathcal{P}[\tilde u](\bm{x}_j)\right|
    =\!O\left(\frac{(\log N)^{\alpha s}}{N^\alpha}\right)\!\!.
  \end{equation}
\end{theorem}
Intuitively, if $\mathcal{P}[\tilde{u}]$ satisfies certain conditions, we can find a set $L^*$ of collocation points with which the objective function \eqref{eq:dloss} approximates the integral \eqref{eq:closs} only within an error of $O(\frac{(\log N)^{\alpha s}}{N^\alpha})$.
This rate is much better than that of the uniformly random sampling (i.e., the Monte Carlo method), which is of $O(1/N^{\frac{1}{2}})$~\cite{Sloan1994-cl}, if the activation function of $\tilde{u}$ and hence the neural network $\tilde{u}$ itself are sufficiently smooth so that $\mathcal{P}[\tilde{u}] \in E_\alpha$ for a large $\alpha$.
Hence, in this paper, we call the training method that minimizes \eqref{eq:dloss}
for a lattice $L$ satisfying \eqref{eq:goodlattice} the \textit{good lattice training (GLT)}.

While any set of collocation points that satisfies the above condition will have the same convergence rate, a set constructed by the vector $\bm{z} \in \mathbb{Z}^s$ that minimizes \eqref{eq:palpha} leads to better accuracy.
When $s=2$, it is known that a good lattice can be constructed by setting $N = F_k$ and $\bm{z} = (1, F_{k-1})$, where $F_k$ denotes the $k$-th Fibonacci number~\cite{Niederreiter1992-qb, Sloan1994-cl}.
In general, an algorithm exists that can determine the optimal $\bm{z}$ with a computational cost of $O(N^2)$.
See Appendix ``Theoretical Background'' for more details.
Also, we can retrieve the optimal $\bm{z}$ from numerical tables found in references, such as \citet{Fang1994,Keng1981}.

\paragraph{Periodization and Randomization Tricks}
The integrand $\mathcal{P}[\tilde{u}]$ of the loss function \eqref{eq:closs} does not always belong to the Korobov space $E_\alpha$ with high smoothness $\alpha$.
To align the proposed GLT with theoretical expectations, we propose periodization tricks for ensuring periodicity and smoothness.

Given an initial condition at time $t=0$, the periodicity is ensured by extending the lattice twice as much along the time coordinate and folding it.
Specifically, instead of $t$, we use $\hat t$ as the time coordinate that satisfies $t=2\hat t$ if $\hat t<0.5$ and $t=2(1-\hat t)$ otherwise (see the lower right panel of Fig.~\ref{fig:collocation}, where the time is put on the horizontal axis).
Also, while not mandatory, we combine the initial condition $u_0(\bm{x}_{\!/t})$ and the neural network's output $\tilde{u}(t,\dots)$ as $\exp(-t)u_0(\bm{x}_{\!/t})+(1-\exp(-t))\tilde{u}(t,\bm{x}_{\!/t})$, thereby ensuring the initial condition, where $\bm{x}_{\!/t}$ denotes the set of coordinates except for the time coordinate $t$.
A similar idea was proposed in \citet{Lagaris1998}, which however does not ensure the initial condition strictly.
If a periodic boundary condition is given to the $k$-th space coordinate, we bypass learning it and instead map the coordinate $x_k$ to a unit circle in two-dimensional space.
Specifically, we map $x_k$ to $(x_k^{(1)}, x_k^{(2)}) = (\cos(2\pi x_k), \sin(2\pi x_k))$, assuring the loss function $\mathcal{P}[\tilde u]$ to take the same value at the both edges ($x_k=0$ and $x_k=1$) and be periodic.
Given a Dirichlet boundary condition $u = 0$ at $\partial\Omega$ to the $k$-th axis, we multiply the neural network's output $\tilde{u}(\dots,x_k,\dots)$ by $x_k(1 - x_k)$, and treat the result as the approximated solution.
This ensures the Dirichlet boundary condition is met, and the loss function $\mathcal{P}[\tilde u]$ takes zero at the boundary $\partial\Omega$, thereby ensuring the periodicity.
If a more complicated Dirichlet boundary condition is given, one can fold the lattice along the space coordinate in the same manner as the time coordinate and ensure the periodicity of the loss function $\mathcal{P}[\tilde u]$.

These periodization tricks aim to satisfy the periodicity conditions necessary for GLT to exhibit the performance shown in Theorem~\ref{theorem:goodlattice}.
However, they are also available for other sampling methods and potentially improve the practical performance by liberating them from the effort of learning initial and boundary conditions.

Since the GLT is grounded on the Fourier series foundation, it allows for the periodic shifting of the lattice.
Hence, we randomize the collocation points as
\begin{equation*}
  \textstyle L^*=  \{ \mbox{the decimal part of\ } \frac{j}{N} \bm{z}+\vr  \mid j=0,\ldots,N-1\},
  \label{eq:goodlatticeimplementation}
\end{equation*}
where $\vr$ follows the uniform distribution over the unit cube $[0, 1]^s$.
Our preliminary experiments confirmed that, if using the stochastic gradient descent (SGD) algorithm, resampling the random numbers $\vr$ at each training iteration prevents the neural network from overfitting and improves training efficiency.
We call this approach the randomization trick.

\paragraph{Limitations}
Not only is this true for the proposed GLT, but most strategies to determine collocation points are not directly applicable to non-rectangular or non-flat domain $\Omega$~\cite{Shankar2018}.
To achieve the best performance, the PDEs should be transformed to such a domain by an appropriate coordinate transformation.
See \citet{Knupp2020-ix, Thompson1985-tx} for examples.

Previous studies on numerical analysis addressed the periodicity and smoothness conditions on the integrand by variable transformations~\cite{Sloan1994-cl} (see also Appendix ``Theoretical Background'').
However, our preliminary experiments confirmed that it did not perform optimally in typical problem settings.
Intuitively, these variable transformations reduce the weights of regions that are difficult to integrate, suppressing the discretization error.
This implies that, when used for training, the regions with small weights remain unlearned.
As a viable alternative, we introduced the periodization tricks to ensure periodicity.

The performance depends on the smoothness of the physics-informed loss, and hence on the smoothness of the neural network and the true solution.
See Appendix ``Theoretical Background'' for details.

\section{Experiments and Results}
\paragraph{Physics-Informed Neural Networks}
We modified the code from the official repository\footnote{\url{https://github.com/maziarraissi/PINNs} (MIT license)} of \citet{Raissi2019}, the original paper on PINNs.
We obtained the datasets of the nonlinear Schr\"{o}dinger (NLS) equation, Korteweg--De Vries (KdV) equation, and Allen-Cahn (AC) equation from the repository.
The NLS equation governs wave functions in quantum mechanics, while the KdV equation models shallow water waves, and the AC equation characterizes phase separation in co-polymer melts.
These datasets provide numerical solutions to initial value problems with periodic boundary conditions.
Although they contain numerical errors, we treated them as the true solutions $u$.
These equations are nonlinear versions of hyperbolic or parabolic PDEs.
Additionally, as a nonlinear version of an elliptic PDE, we created a dataset for $s$-dimensional Poisson's equation, which produces analytically solvable solutions with $2^s$ modes with the Dirichlet boundary condition.
We examined the cases where $s\in\{2,4\}$.
See Appendix ``Experimental Settings'' for further information.

Unless otherwise stated, we followed the repository's experimental settings for the NLS equation.
The physics-informed loss was defined as $\mathcal P[\tilde{u}]=\frac{1}{N}\sum_{j=0}^{N-1}\|\mathcal{N}[\tilde{u}](\bm{x}_j)\|^2$ given $N$ collocation points $\{\bm{x}_j\}_{j=0}^{N-1}$.
This can be regarded as a finite approximation to the squared 2-norm $|\mathcal{N}[\tilde{u}]|_2^2=\int_{\Omega}\|\mathcal{N}[\tilde{u}](\bm{x})\|^2\mathrm{d}\bm{x}$.
The state of the NLS equation is complex; we simply treated it as a 2D real vector for training and used its absolute value for evaluation and visualization.
Following \citet{Raissi2019}, we evaluated the performance using the relative error, which is the normalized squared error
$\mathcal L(\tilde{u}, u;\bm{x}_j)=(\sum_{j=0}^{N_e-1}\|\tilde{u}(\bm{x}_j)-u(\bm{x}_j)\|^2)^{\frac{1}{2}}/(\sum_{j=0}^{N_e-1}\|u(\bm{x}_j)\|^2)^{\frac{1}{2}}$
at predefined $N_e$ collocation points $\{\bm{x}_j\}_{j=0}^{N_e-1}$.
This is also a finite approximation to $|\tilde{u}-u|_2/|u|_2$.

We applied the above periodization tricks to each test and model for a fair comparison.
For Poisson's equation with $s=2$, which gives the exact solutions, we followed the original learning strategy using the L-BFGS-B method preceded by the Adam optimizer~\cite{Kingma2014b} for 50,000 iterations to ensure precise convergence.
For other datasets, which contain the numerical solutions, we trained PINNs using the Adam optimizer with cosine decay of a single cycle to zero~\cite{Loshchilov2017} for 200,000 iterations and sampled a different set of collocation points at each iteration.
See Appendix ``Experimental Settings'' for details.

We determined the collocation points using uniformly random sampling, uniformly spaced sampling, LHS, the Sobol sequence, and the proposed GLT.
For the GLT, we took the number $N$ of collocation points and the corresponding integer vector $\bm{z}$ from numerical tables in \cite{Fang1994,Keng1981}.
We used the same values for uniformly random sampling and LHS to maintain consistency.
For uniformly spaced sampling, we selected numbers $N=m^s$ for $m\in\mathbb{N}$ that were closest to a number used in the GLT, creating a unit cube of $m$ points on each side.
We additionally applied the randomization trick.
For the Sobol sequence, we used $N=2^m$ for $m\in\mathbb N$ due to its theoretical background.
We conducted five trials for each number $N$ and each method.
All experiments were conducted using Python v3.7.16 and tensorflow v1.15.5~\cite{tensorflow} on servers with Intel Xeon Platinum 8368.

\begin{figure}
  \centering
  \includegraphics[scale=0.8]{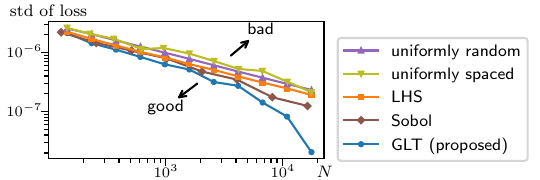}
  \caption{The number $N$ of collocation points and the standard deviation of the physics-informed loss, which approximates the discretization error $|\eqref{eq:closs}-\eqref{eq:dloss}|$.}
  \label{fig:PINNsTest}
\end{figure}

\begin{table*}[t]
  \centering
  \caption{Trade-Off between Number $N$ of Collocation Points and Relative Error $\mathcal L$.}
  \label{tab:datasets}
  \scriptsize
  \renewcommand{\arraystretch}{0.9}
  \begin{tabular}{c@{ }lrrrrrrrrrr}
    \toprule
                                                                                     &                  & \multicolumn{5}{c}{\# of points $N^\dagger$} & \multicolumn{5}{c}{relative error $\mathcal L^\ddagger$}                                                                                                                                                                                                   \\
    \cmidrule(lr){3-7}\cmidrule(lr){8-12}
                                                                                     &                  & \multicolumn{1}{c}{NLS}                      & \multicolumn{1}{c}{KdV}                                  & \multicolumn{1}{c}{AC} & \multicolumn{2}{c}{Poisson} & \multicolumn{1}{c}{NLS} & \multicolumn{1}{c}{KdV} & \multicolumn{1}{c}{AC} & \multicolumn{2}{c}{Poisson}                                 \\
    \cmidrule(lr){6-7}\cmidrule(lr){11-12}
                                                                                     &                  &                                              &                                                          &                        & $s=2$                       & $s=4$                   &                         &                        &                             & $s=2$         & $s=4$         \\
    \midrule
    \scalebox{1.0}[0.8]{\textcolor{C4}{$\blacktriangle$}}                            & uniformly random & $>$4,181                                     & $>$4,181                                                 & 4,181                  & $>$4,181                    & 1,019                   & 3.11                    & 2.97                   & 1.55                        & 28.53         & 0.28          \\
    \rotatebox[origin=c]{180}{\scalebox{1.0}[0.8]{\textcolor{C8}{$\blacktriangle$}}} & uniformly spaced & 2,601                                        & 4,225                                                    & $>$4,225               & $>$4,225                    & $>$4,096                & 2.15                    & 3.28                   & 1.95                        & 5.16          & 1437.12       \\
    \scalebox{0.7}{\textcolor{C1}{$\blacksquare$}}                                   & LHS              & $>$4,181                                     & 4,181                                                    & 4,181                  & 4,181                       & 701                     & 2.75                    & 3.06                   & 1.25                        & 246.29        & 0.24          \\
    \rotatebox[origin=c]{45}{\scalebox{0.6}{\textcolor{C5}{$\blacksquare$}}}         & Sobol            & 2,048                                        & 2,048                                                    & 4,096                  & $>$4,096                    & 1,024                   & 2.05                    & 2.52                   & 1.22                        & 14.74         & 1.22          \\
    \midrule
    \textcolor{C0}{$\bullet$}                                                        & GLT (proposed)   & \textbf{987}                                 & \textbf{987}                                             & \textbf{1,597}         & \textbf{610}                & \textbf{307}            & \textbf{1.22}           & \textbf{2.19}          & \textbf{0.93}               & \textbf{0.76} & \textbf{0.15} \\
    \bottomrule                                                                                                                                                                                                                                                                                                                                                                                                     \\[-3mm]
    \multicolumn{11}{l}{$\dagger$ \# of points $N$ at competitive relative error $\mathcal L$ (under horizontal red line in Fig.~\ref{fig:PINNsResults}).}                                                                                                                                                                                                                                                          \\
    \multicolumn{11}{l}{$\ddagger$ relative error $\mathcal L$ at competitive \# of points $N$ (on vertical green line in Fig.~\ref{fig:PINNsResults}). Shown in the scale of $10^{-3}$.}                                                                                                                                                                                                                           \\
  \end{tabular}
\end{table*}

\begin{figure*}[t]
  \tabcolsep=1mm
  \centering
  \scriptsize
  \includegraphics[scale=0.8]{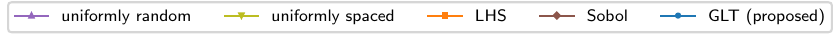}                  \\
  \begin{tabular}{cccccc}
    \hspace*{-2mm}\includegraphics[scale=0.8]{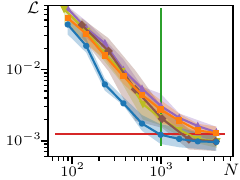}       &
    \hspace*{-2mm}\includegraphics[scale=0.8]{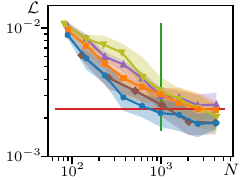}       &
    \hspace*{-2mm}\includegraphics[scale=0.8]{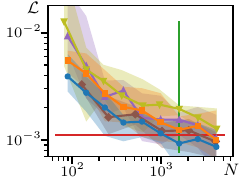}        &
    \hspace*{-2mm}\includegraphics[scale=0.8]{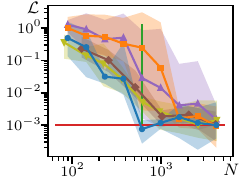} &
    \hspace*{-2mm}\includegraphics[scale=0.8]{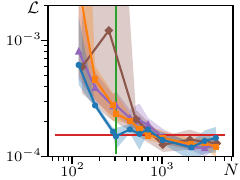}   \\[-1.5mm]
    \hspace*{4mm} NLS                                                    &
    \hspace*{4mm} KdV                                                    &
    \hspace*{4mm} AC                                                     &
    \hspace*{4mm} Poisson with $s=2$                                     &
    \hspace*{4mm} Poisson with $s=4$                                       \\[-2mm]
  \end{tabular}
  \caption{The results of PINNs.
    The number $N$ of collocation points and the relative error $\mathcal L$.}
  \label{fig:PINNsResults}
  \vspace*{3mm}
  \tabcolsep=1mm
  \centering
  \scriptsize
  \begin{tabular}{ccccc}
    \hspace*{-2mm}\includegraphics[scale=0.8]{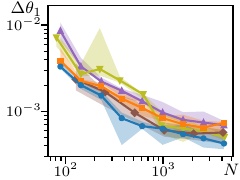} &
    \hspace*{-2mm}\includegraphics[scale=0.8]{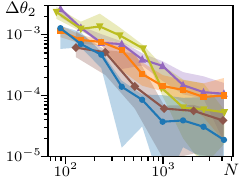} &
    \hspace*{-2mm}\includegraphics[scale=0.8]{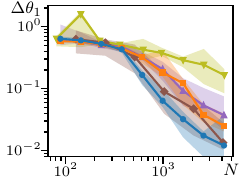}  &
    \hspace*{-2mm}\includegraphics[scale=0.8]{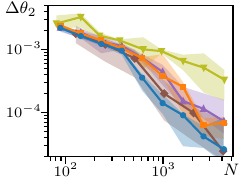}    \\[-1.5mm]
    \hspace*{4mm} KdV $\theta_1$                                             &
    \hspace*{4mm} KdV $\theta_2$                                             &
    \hspace*{4mm} AC $\theta_1$                                              &
    \hspace*{4mm} AC $\theta_2$ \\[-2mm]
  \end{tabular}
  \caption{The results of system identification.
    The number $N$ of collocation points and the relative error $\Delta\theta$ of the learnable parameter $\theta$.
    The legend can be found in Fig.~\ref{fig:PINNsResults}.
  }
  \label{fig:identification}
\end{figure*}
\paragraph{Results: Accuracy of Physics-Informed Loss}
First, we evaluate the accuracy of the physics-informed loss \eqref{eq:dloss} in approximating the integrated loss \eqref{eq:closs}.
Note that it is intractable to obtain the integrated loss \eqref{eq:closs}, which motivates the present study.
Instead, we use the standard deviation of the physics-informed loss \eqref{eq:dloss} as an approximation to the discretization error $|\eqref{eq:closs}-\eqref{eq:dloss}|$.
This is because the average of the physics-informed loss \eqref{eq:dloss} is assumed to converge to the integrated loss \eqref{eq:closs}, and the standard deviation represents the average error.
We trained the PINNs on the NLS equation with $N=610$ collocation points determined by LHS.
Then, we evaluated the physics-informed loss \eqref{eq:dloss} for different methods and different numbers of collocation points.
For each combination of method and number, we performed 10,000 trials and summarized their results in Fig.~\ref{fig:PINNsTest}.

Except for the Sobol sequence and GLT, the other methods exhibit a similar trend, showing a linear reduction on the log-log plot.
This trend aligns with the theoretical result that the convergence rate $O(1/N^{\frac{1}{2}})$ of the Monte Carlo method and that $O(1/N^{\frac{1}{s}})$ of the uniformly spaced sampling.
The Sobol sequence shows a slightly faster reduction, and the GLT demonstrates a further accelerated reduction as the number $N$ increases.
This result implies that by using the GLT, the physics-informed loss \eqref{eq:dloss} approximates the integrated loss \eqref{eq:closs} more accurately with the same number $N$ of collocation points, leading to faster training and improved accuracy.

The Sobol sequence produces the discretization error of $O(\frac{(\log N)^{s}}{N})$ for smooth solutions $u$ and neural networks, which is smaller than that $O(1/N^{\frac{1}{2}})$ of the Monte Carlo method for a large number $N$~\cite{Lye2020,Mishra2021}.
As shown in Theorem~\ref{theorem:goodlattice}, the proposed GLT produces the discretization error of $O(\frac{(\log N)^{\alpha s}}{N^\alpha})$, which is comparable to the Sobol sequence for solutions with $\alpha=1$ and is much smaller for smoother solutions with $\alpha>1$.
These are several higher-order quasi-Monte Carlo methods, which potentially suppress the discretization errors for smooth solutions with $\alpha>1$.
However, these methods require the prior knowledge about the smoothness $\alpha$ and careful adjustments of hyperparameters~\cite{Longo2021}.
On the other hand, the proposed GLT is free from these prior knowledge or adjustments and achieves better performances depending on the smoothness $\alpha$ of the solution $u$ and the neural network, which is a significant advantage.

\paragraph{Results: Performance of PINNs}
Figure~\ref{fig:PINNsResults} shows the average relative error $\mathcal L$ using solid lines, with the maximum and minimum errors depicted by shaded areas.
As $N$ increases, the relative error $\mathcal L$ decreases and eventually reaches saturation.
This saturation is attributed to several factors: the limitations in the network architecture as mentioned in Theorem \ref{thm:3.1}, numerical errors in the datasets, discretization errors in relative error $\mathcal L$, and rounding errors in computation.

We report the minimum numbers $N$ of collocation points with which the relative error $\mathcal L$ was saturated in the left half of Table~\ref{tab:datasets}.
Specifically, we consider a relative error $\mathcal L$ below 130 \% of the minimum observed one as saturated; the thresholds are denoted by horizontal red lines in Fig.~\ref{fig:PINNsResults}.
The proposed GLT exhibited competitive performances with considerably fewer collocation points.
Specifically, it required less than half the number of points as compared to the second-best methods, and in the case of $s=2$-dimensional Poisson's equation, it needed only a seventh of the points.
These findings indicate that the proposed GLT can reduce computational costs significantly.
Note that, for $s=4$-dimensional Poisson's equation, the performance of the uniformly spaced sampling was extremely inferior, resulting in it not being captured in the image.

Following this, we standardized the number $N$ of collocation points (and hence, the computational cost).
In the right half of Table~\ref{tab:datasets}, we list the relative error $\mathcal L$ observed for collocation points where the relative error $\mathcal L$ of one of the comparison methods reached saturation, as denoted by vertical green lines in Fig.~\ref{fig:PINNsResults}.
A smaller error $\mathcal L$ indicates that a method outperforms others at the same computational cost.
The proposed GLT yielded the smallest relative errors $\mathcal L$ across all cases, with a particularly pronounced difference in Poisson's equation with $s=2$.
We show the true solutions and the residuals of example results with such $N$ in Fig.~\ref{fig:PINNsExample} in Appendix ``Additional Results.''

Therefore, we conclude that the proposed GLT can solve various PDEs with better performances and fewer collocation points, provided the dimension number $s$ of the domain $\Omega$ is four or less---a range adequate for most physical simulations.
We also confirmed that the periodization tricks significantly improve the overall performance in Appendix ``Additional Results.''
Refer to Appendix ``Higher Dimensional Case'' for higher-dimensional cases.

\paragraph{Results: System Identification}
We further assessed the performance of white-box system identification, in a similar way as in those in \citet{Raissi2019}.
For the KdV and AC equations, we treated two parameters, $(\theta_1,\theta_2)$, as learnable parameters $(\tilde\theta_1,\tilde\theta_2)$ and initialized them to zero.
We extracted the true solutions $u$ at $N_s$ randomly selected points $\{\bm{x}_j\}_{j=0}^{N_s-1}$ as observations.
During the training, in addition to the physics-informed loss, we minimized the mean squared error of the state $\tilde u$ at these points, that is, $\frac{1}{N_s}\sum_{j=0}^{N_s-1} \|u(\bm{x}_j)-\tilde u(\bm{x}_j)\|^2$.
This procedure guides the learnable parameters $(\tilde\theta_1,\tilde\theta_2)$ to the true values $(\theta_1,\theta_2)$.
We set the number of points $N_s$ to 100 for the KdV equation and 200 for the AC equation, which were nearly the minimum required for successful system identification.
All other experimental settings were identical to those in the previous experiments.

Figure~\ref{fig:identification} shows the median of the relative error $\Delta\theta_1=|\tilde\theta_1-\theta_1|/|\theta_1|$ for the five trials.
Our proposed GLT demonstrated the highest precision.
Remarkably, it achieved comparable accuracy with approximately half the number of collocation points $N$ required for the Sobol sequence in most cases, and significantly fewer than other methods.
Recall that the observations $\{\bm{x}_j\}_{j=0}^{N_s-1}$ were selected using the Monte Carlo method, leading to the error in the order of $O(1/N_s^{\frac{1}{2}})$ for all methods.
Nonetheless, the accuracy of parameter identification relies heavily on the strategy to determine the collocation points for the physics-informed loss and its number $N$.
The proposed GLT method proved to be superior in this aspect as well.

\paragraph{Competitive Physics-Informed Neural Networks}
Competitive PINNs (CPINNs) are an improved version of PINNs with an additional neural network $D:\Omega\rightarrow\R$ called a discriminator~\cite{Zeng2023}.
Its objective function is $\frac{1}{N}\sum_{j=0}^{N-1} D(\bm{x}_j)\mathcal{N}[\tilde{u}](\bm{x}_j)$; the discriminator $D$ is trained to maximize it, whereas the neural network $\tilde{u}$ is trained to minimize it, forming a zero-sum game.
The Nash equilibrium offers the solution to a given PDE.
CPINNs employed the competitive gradient descent algorithm to accelerate the convergence~\cite{Schaefer2019,Schaefer2020}.
The objective function is also regarded as a finite approximation to the integral $\int_{\bm{x}\in\Omega}D(\bm{x})\mathcal{N}[\tilde{u}](\bm{x}) \mathrm{d}\bm{x}$; therefore, the proposed GLT in applicable to CPINNs.

We modified the code accompanying the manuscript and investigated the NLS and Burgers' equations\footnote{See Supplementary Material at \url{https://openreview.net/forum?id=z9SIj-IM7tn} (MIT License)\label{footenote:CPINNs}}.
The NLS equation is identical to the one above.
See Appendix ``Experimental Settings'' for details about Burgers' equation.
The number $N$ of collocation points was 20,000 by default and varied.
We folded the coordinates to ensure the periodicity of the loss function for the proposed GLT but did not ensure the initial and boundary conditions; we trained neural networks to learn the initial and boundary conditions following the original experimental settings.
Also, we did not apply the randomization trick.
All experiments were conducted using Python v3.9.16 and Pytorch v1.13.1~\cite{Paszke2017} on servers with Intel Xeon Platinum 8368 and NVIDIA A100.

\begin{figure}[t]
  \centering
  \tabcolsep=1mm
  \scriptsize
    \includegraphics[scale=0.65]{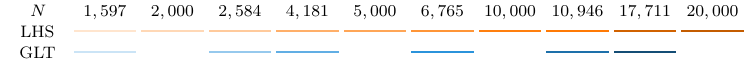}
\begin{tabular}{cc}
    \includegraphics[scale=0.65]{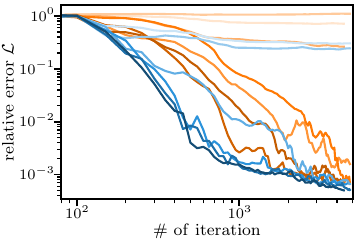} &
    \includegraphics[scale=0.65]{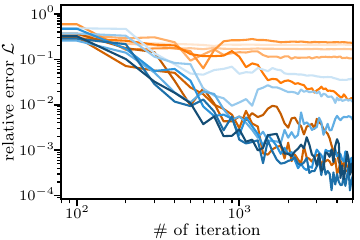} \\[-1mm]
    NLS                                                      &
    Burgers                                                      \\
  \end{tabular}
  \caption{The results of CPINNs.
    The number of iterations and the relative error $\mathcal L$.}
  \label{fig:CPINNsResults}
\end{figure}

\paragraph{Results}
We have summarized the results in Fig.~\ref{fig:CPINNsResults}.
In the case of the NLS equation, using LHS, the relative error $\mathcal L$ declines slowly even when $N=20,000$, and there was almost no improvement for $N\le 4,181$.
Conversely, using GLT rapidly reduces the relative error $\mathcal L$ at $N=6,765$.
In most cases, GLT requires only one-third of the collocation points that LHS needs to achieve a comparable level of performance.
In the case of the Burgers' equation, CPINNs using GLT demonstrated progress in learning with $N=2,584$ collocation points, whereas CPINNs using LHS method with $N=10,946$ achieved a worse performance rate.
These results indicate that the proposed GLT exhibits competitive or superior convergence speed with 3 to 4 times fewer collocation points.
The original paper demonstrated that CPINNs have a faster convergence rate than vanilla PINNs, but the GLT can further accelerate it.

\section{Conclusion}
This paper highlighted that the physics-informed loss, commonly used in PINNs and their variants, is a finite approximation to the integrated loss.
From this perspective, we proposed good lattice training (GLT) to determine collocation points.
This method enables a more accurate approximation of the integrated loss with a smaller number of collocation points.
Experimental results using PINNs and CPINNs demonstrated that the GLT can achieve competitive or superior performance with much fewer collocation points.
These results imply a significant reduction in computational cost and contribute to the large-scale computation of PINNs.

As shown in Figs.~\ref{fig:PINNsResults} and \ref{fig:CPINNsResults}, the current problem setting reaches performance saturation with around $N=1,000$ collocation points due to the network capacity and numerical errors in the datasets.
However, Figure~\ref{fig:PINNsTest} demonstrates that the GLT significantly enhances the approximation accuracy even when using many more collocation points.
This implies that the GLT is particularly effective in addressing larger-scale and high-precision problem settings, which will be explored further in future research.

\section{Acknowledgments}
This study was partially supported by JST CREST (JPMJCR1914), JST PRESTO (JPMJPR21C7), JST ASPIRE (JPMJAP2329), JST Moonshot R\&D (JPMJMS2033-14), and JSPS KAKENHI (24K15105), and was achieved through the use of SQUID at D3 Center, Osaka University.


\newpage
\appendix
\renewcommand\thetable{A\arabic{table}}
\setcounter{table}{0}
\renewcommand\thefigure{A\arabic{figure}}
\setcounter{figure}{0}

{\Huge Technical Appendix}
\section{Theoretical Background}
\subsection{Proof of Theorem \ref{thm:3.1}}
In this section, we prove Theorem \ref{thm:3.1}:
\begin{theorem*}
  Suppose that the class of neural networks used for PINNs includes an $\varepsilon_1$-approximator $\tilde{u}_\mathrm{opt}$ to the exact solution $u^*$ to the PDE $\mathcal{N}[u]=0$:
  \begin{align*}
    \|u^* - \tilde{u}_\mathrm{opt} \| \leq \varepsilon_1,
  \end{align*}
  and that \eqref{eq:dloss} is an $\varepsilon_2$-approximation of $\eqref{eq:closs}$ for the approximated solution $\tilde{u}$ and also for $\tilde{u}_\mathrm{opt}$:
  \begin{align*}
    |
    \textstyle \int_{[0, 1]^s} \mathcal{P}[{u}](\bm{x}) \mathrm{d}\bm{x} - \frac{1}{N}\sum_{\vx_j\in L^*} \mathcal{P}[u](\bm{x}_j)| \leq \varepsilon_2
  \end{align*}
  for $u=\tilde{u}$ and $u=\tilde{u}_\mathrm{opt}$.
  Suppose also that there exist a constant $c_\mathrm{p} > 0$ such that
  \begin{align*}
    \|u - v\| \leq c_\mathrm{p} \| \mathcal{N}[u] - \mathcal{N}[v] \|
  \end{align*}
  and
  a constant $c_\mathrm{L} > 0$ such that
  \begin{align*}
    \|\mathcal{N}[u] - \mathcal{N}[v]\| \leq c_\mathrm{L} \| u - v \|.
  \end{align*}
  Then, the numerical error of $\tilde{u}$ is estimated by
  \begin{equation*}
    \textstyle \|u^* - \tilde{u}\|
    \leq (1 + c_\mathrm{p} c_\mathrm{L}) \varepsilon_1
    + c_\mathrm{p} \sqrt{\frac{1}{N}\sum_{\vx_j\in L^*} \mathcal{P}[\tilde u](\bm{x}_j) + \varepsilon_2}.
  \end{equation*}
\end{theorem*}

\begin{proof}
  First, we have
  \begin{align*}
    \|u^* - \tilde{u}\|
     & = \|u^* - \tilde{u}_\mathrm{opt} + \tilde{u}_\mathrm{opt} - \tilde{u}\|                                                 \\
     & \leq \|u^* - \tilde{u}_\mathrm{opt}\| + \|\tilde{u}_\mathrm{opt} - \tilde{u}\|                                          \\
     & \leq \|u^* - \tilde{u}_\mathrm{opt}\| + c_\mathrm{p} \| \mathcal{N}[\tilde{u}_\mathrm{opt}] - \mathcal{N}[\tilde{u}] \| \\
     & \leq \varepsilon_1 + c_\mathrm{p} \| \mathcal{N}[\tilde{u}_\mathrm{opt}] - \mathcal{N}[\tilde{u}] \|.
  \end{align*}
  Since $u^*$ is the exact solution, $\mathcal{N}[u^*] = 0$. Thus we have
  \begin{align*}
    \| \mathcal{N}[\tilde{u}_\mathrm{opt}] - \mathcal{N}[\tilde{u}] \|
     & = \|- \mathcal{N}[u^*] + \mathcal{N}[\tilde{u}_\mathrm{opt}] - \mathcal{N}[\tilde{u}] \|
    \\ &
    \leq \|\mathcal{N}[u^*] - \mathcal{N}[\tilde{u}_\mathrm{opt}]\| + \|\mathcal{N}[\tilde{u}] \|
    \\ & \leq c_\mathrm{L} \|u^* - u_\mathrm{opt}\| + \|\mathcal{N}[\tilde{u}] \|
    \\ & \leq c_\mathrm{L} \varepsilon_1 + \|\mathcal{N}[\tilde{u}] \|
  \end{align*}
  because $\mathcal{N}$ is assumed to be Lipschitz continuous.
  From the assumption, we have
  \begin{equation*}
  \begin{aligned}
    \|\mathcal{N}[\tilde{u}] \|^2 & = \int_{[0, 1]^s} \mathcal{P}[\tilde{u}](\bm{x}) \mathrm{d}\bm{x}
    \\ &\textstyle
    \leq \frac{1}{N}\sum_{\vx_j\in L^*} \mathcal{P}[\tilde u](\bm{x}_j)
    \\ &\textstyle
    \quad +
    \left|\int_{[0, 1]^s} \mathcal{P}[\tilde{u}](\bm{x}) \mathrm{d}\bm{x} - \frac{1}{N}\sum_{\vx_j\in L^*} \mathcal{P}[\tilde u](\bm{x}_j) \right|
    \\ &\textstyle
    \leq \frac{1}{N}\sum_{\vx_j\in L^*} \mathcal{P}[\tilde u](\bm{x}_j) + \varepsilon_2.
  \end{aligned}
  \end{equation*}
  Thus, we have
  \begin{equation*}
    \textstyle \|u^* - \tilde{u}\|
    \leq (1 + c_\mathrm{p} c_\mathrm{L}) \varepsilon_1
    + c_\mathrm{p} \sqrt{\frac{1}{N}\sum_{\vx_j\in L^*} \mathcal{P}[\tilde u](\bm{x}_j) + \varepsilon_2}.
  \end{equation*}
\end{proof}

\subsection{How to Find Good Lattice}
It is known that for $s=2$, $\bm{z} \in \mathbb{Z}^2$ in Theorem~\ref{theorem:goodlattice} can be constructed by using the Fibonacci sequence~\cite{Niederreiter1992-qb, Sloan1994-cl}. Specifically,
\begin{equation*}
  \begin{aligned}
    F_1 = F_2 = 1, \quad F_{k} = F_{k-1} + F_{k-2} \ (k \geq 3),
  \end{aligned}
\end{equation*}
$\bm{z} = (1, F_{k-1})^\top$ with $N= F_k$.
It is known that $\bm{h} \cdot \bm{z} \equiv 0 \pmod N$ gives a small $\bar{h}_1 \bar{h}_2$, making $P_\alpha(\bm{z}, N)$ large.
Hence, it is preferable to choose $\bm{z}$ so that $\bm{z}$  maximizes a minimized $\bar{h}_1 \bar{h}_2$:
\begin{equation*}
  \max_{\bm{z} \in \mathbb{Z}^2} \quad \min_{\bm{h} \in \mathbb{Z}^2, \bm{h} \neq 0, \bm{h} \cdot \bm{z} \equiv 0 \pmod N} \bar{h}_1 \bar{h}_2.
\end{equation*}
Because $\{ F_k \}$ is the Fibonacci sequence, it holds that for $\bm{h} = (F_{k-2}, 1)^\top$
\begin{equation*}
  \begin{aligned}
    \bm{h} \cdot \bm{z}
     & =  \begin{pmatrix} F_{k-2} \\ 1 \end{pmatrix}
    \cdot
    \begin{pmatrix} 1 \\ F_{k-1} \end{pmatrix}           \\
     & = F_{k-2} + F_{k-1} = F_{k} = N \equiv 0 \pmod N.
  \end{aligned}
\end{equation*}
This $\bm{h}$ gives $\bar{h}_1 \bar{h}_2 = F_{k-2}$. 
It is known that this achieves a small enough $\bar{h}_1 \bar{h}_2$.
Since $F_k$ is written as
\begin{equation*}
  F_k = \frac{1}{\sqrt{5}} \left(
  \left(\frac{1+\sqrt{5}}{2}\right)^k - \left(\frac{1-\sqrt{5}}{2}\right)^k
  \right),
\end{equation*}
$F_k/F_{k-2} \to ((1+\sqrt{5})/2)^2$ as $k \to \infty$. Hence, if $N=F_k$, then $\bar{h}_1 \bar{h}_2 = F_{k-2} \to N/((1+\sqrt{5})/2)^2$ as $k \to \infty$. Hence, roughly, $1/ (\bar{h}_1 \bar{h}_2)^\alpha$ in $P_\alpha(\bm{z}, N)$ becomes $O(1/N^\alpha)$.
See~\citet{Niederreiter1992-qb, Sloan1994-cl} for more strict proof.

The following explains why the Fibonacci sequence is better from a more number-theoretic point of view. For an integer $N \geq 2$, let $\bm{z} = (1, z_2) \in \mathbb{Z}^2$ with $\gcd(z_2, N)=1$. Let the continued fraction expansion of the rational number $z_2/N$ be
\begin{equation*}
  \frac{z_2}{N}=a_0 + 1/(a_1 + 1/(a_2 + \cdots) ),
\end{equation*}
where $a_0 = \lfloor \frac{z_2}{N} \rfloor$ and $a_i \in \mathbb{N}$ for $1\le i\le k$.
Let $r_i$ be the rational number obtained by truncating this expansion up to $a_i$. It can be confirmed that $r_i$ can be written as $r_i = p_i/q_i$ using $q_i, p_i \in \mathbb{Z}$, obtained as follows.
\begin{equation*}
  \begin{aligned}
     & p_{0} = a_0, p_{1}=a_0 a_1 + 1, p_i = a_i p_{i-1}+p_{i-2}, \\
     & q_{0} = 1, q_{1}=a_1, q_i = a_i q_{i-1}+q_{i-2}.
  \end{aligned}
\end{equation*}
It is also confirmed that this gives the reduced form: $\gcd(p_i, q_i )=1$~\cite{Niederreiter1992-qb}.

Such a continued fraction expansion is used in the Diophantine approximation. The Diophantine approximation is one of the problems in number theory, which includes the problem of approximating irrational numbers by rational numbers. The approximations are classified into ``good'' approximation and ``best'' approximation according to the degree of the approximation, and it is known that the above continued fraction expansion gives a ``good'' approximation.
This is why the lattice $L$ is called a ``good'' lattice.
The accuracy of the approximation has also been studied, and it is known that in the above case, the following holds~\cite{Niederreiter1992-qb}:
\begin{equation*}
  \begin{aligned}
    \frac{1}{q_i (q_i + q_{i+1})} \leq \left| \frac{z_2}{N}  - \frac{p_i}{q_i}\right| \leq \frac{1}{q_i q_{i+1}},
  \end{aligned}
\end{equation*}
from which it follows \cite{Niederreiter1992-qb} that
\begin{equation*}
  \begin{aligned}
    \frac{N}{\max_{1 \leq i \leq k} a_i + 2}
     & \leq \min_{\bm{h} \in \mathbb{Z}^2, \bm{h} \neq 0, \bm{h} \cdot \bm{z} \equiv 0 \pmod N} \bar{h}_1 \bar{h}_2 \\
     & \hspace*{2cm} \leq \frac{N}{\max_{1 \leq i \leq k} a_i}.
  \end{aligned}
\end{equation*}
Hence, to maximize $\min_{\bm{h}} \bar{h}_1 \bar{h}_2$, it is preferable to use the pair of $z_2$ and $N$ such that $\max_{1 \leq i \leq k} a_i$ is as small as possible. To this end, $a_i$'s should be determined by $a_0 = 0$ and $a_i = 1$ for all $i \geq 1$. Substituting these $a_i$'s into the formulas for $q_i$ and $p_i$ results in
\begin{equation*}
  \begin{aligned}
     & p_0 = 0, p_1 = 1, p_2 = 1, p_3 = 2, \ldots \\
     & q_0 = 1, q_1 = 1, q_2 = 2, q_3 = 3, \ldots
  \end{aligned}
\end{equation*}
which means $p_i = F_{i-1}, q_i = F_{i}$ where $F_i$'s are the Fibonacci sequence. Hence, the pair of $z_2 = F_{k-1}$ and $N = F_k$ gives a ``good'' lattice for computing objective functions.

When $s=2$ and $N$ is not a Fibonacci number or when $s>2$, there is no known method to generate a good lattice using a sequence of numbers.
Even then, we can determine the optimal $\bm{z}$ with a computational cost of $O(N^2)$ as follows:
The upper bound \eqref{eq:interror} is achieved when $\mathcal{P}[\tilde{u}]$ becomes the following function:
\begin{equation*}
  \begin{aligned}
    \mathcal{P}[\tilde{u}](\bm{x}) & = \prod_{k=1}^s F_\alpha(x_k),                                                           \\
    F_\alpha(x_k)                  & = 1 + \sum_{h \in \mathbb{Z}, h \neq 0} \frac{\exp(2 \pi \mathrm{i} h x_k)}{|h|^\alpha},
  \end{aligned}
\end{equation*}
where $x_k$ is the $k$-th element of $\bm{x}$.
Actually, when $\alpha$ is even, this function is known to be explicitly written by using the Bernoulli polynomials $B_\alpha(x)$~\cite{Sloan1994-cl}:
\begin{equation*}
  F_\alpha(x_k) = 1 - (-1)^\frac{\alpha}{2}\frac{(2 \pi)^\alpha B_\alpha}{\alpha!}.
\end{equation*}
Since this is a polynomial function, the integral can be found exactly; hence it is possible to find a loss function for this function for each lattice $L$. Therefore, to find an optimal $\bm{z}$, the loss function with respect to the above function should be minimized.
The computational cost for computing the loss function for each $L$ is $O(N)$.

Suppose each component of $\bm{z}$ is in $\{0, \ldots, N-1\}$.
If $N$ is a prime number or the product of two prime numbers, the existence of a $\bm{z}$ of the form $\bm{z} = (1, l, l^2 \pmod N, \ldots, l^{s-1} \pmod N)$ that gives a good lattice is known \cite{Korobov1959-qd,Korobov1960-lj,Sloan1994-cl}. Since $l \in \{ 0, 1, \ldots, N-1\}$, there are only $N$ candidates for $l$, and hence the computational cost for finding the optimal $\bm{z}$ is only $O(N^2)$ for each $N$.
Furthermore, as this optimization process can be fully parallelized, the computational time is quite short practically.
The values of $\bm{z}$ have been explored in the field of number theoretic numerical analysis, and typical values are available as numerical tables found in references, such as \citet{Fang1994,Keng1981}.

\subsection{Periodization of Integrand}
The method in this paper relies on the Fourier series expansion and hence assumes that the integrand can be periodically extended. In order to extend a given function periodically, it is convenient if the function always vanishes on the boundary. To this end, variable transformations are useful~\cite{Niederreiter1992-qb, Sloan1994-cl}.

For example, suppose that a function $f: [0, 1] \to \mathbb{R}$ does not satisfy $f(0) = f(1)$ and the integral on an interval $[0, 1]$
\begin{equation*}
  \begin{aligned}
    \int_0^1 f(x) \mathrm{d} x
  \end{aligned}
\end{equation*}
must be evaluated as an objective function. Let $y: [0, 1] \to [0, 1]$ be a monotonically increasing smooth map. Then, the change of variables
\begin{equation*}
  \begin{aligned}
    x = y(z) \quad z \in [0, 1]
  \end{aligned}
\end{equation*}
transforms the integral into
\begin{equation*}
  \begin{aligned}
    \int_0^1 f(x) \mathrm{d} x = \int_0^1  f(y(z)) \frac{\mathrm{d} y}{\mathrm{d} z}(z) \mathrm{d}z.
  \end{aligned}
\end{equation*}
Hence, if the derivative $\mathrm{d}y/\mathrm{d}z$ vanishes at $z = 0$ and $z=1$, the integrand becomes periodic. Functions that satisfy these conditions can be easily constructed using polynomials. For example, it is sufficient that this derivative is proportional to $z(1-z)$. Since $z (1-z) \geq 0$ on the interval $[0, 1]$, integration of $z(1-z)$ defines a monotonically increasing smooth function, and hence this can be used to define the function $y(z)$. Smoother transformations can be designed in a similar way by using higher-order polynomials with the constraints such as $\mathrm{d} y/\mathrm{d}z = \mathrm{d}^2 y/\mathrm{d}z^2 = 0$ at $z=0$ and $z=1$~\cite{Sloan1994-cl}.

However, our preliminary experiments confirmed that this periodization left some regions unlearned and exhibited poor performance because it reduces the weights of regions that are difficult to integrate.

\subsection{Smoothness and Performance}
The smoothness parameter $\alpha$ in Theorem~\ref{theorem:goodlattice} is associated with the physics-informed loss and is indeed influenced by the smoothness of the PDE solution itself.
We assume that the neural network $\tilde u$ employs a sufficiently smooth activation function.
When the neural network $\tilde u$ closely approximates a true solution $u$ that is $a$ times differentiable, it is reasonable to expect that the Fourier transform of $\tilde u$ closely resembles that of $u$.
Then, $\tilde u$ is also effectively $a$ times differentiable.
If the PDE $\mathcal{N}[u]=0$ involves derivatives up to the $k$-th order of this solution, and the loss function is based on the squared error (which is at least once differentiable), then it follows that $\alpha=a-k+1$.
Consequently, our method potentially demonstrates reduced performance for PDEs with discontinuities or sharp gradients.

\section{Experimental Settings}
We obtained the datasets of the nonlinear Schr\"{o}dinger (NLS) equation, Korteweg--De Vries (KdV) equation, and Allen-Cahn (AC) equation from the official repository\footnote{\url{https://github.com/maziarraissi/PINNs} (MIT license)} of \citet{Raissi2019}.
These are 2D PDEs, involving 1D space $x$ and 1D time $t$.
The solutions were obtained numerically using spectral Fourier discretization and a fourth-order explicit Runge--Kutta method.
In addition, we created datasets of Poisson's equation.

\paragraph{Nonlinear Schr\"odinger Equation}
A 1D NLS equation is expressed as
\begin{equation}
  \mathrm{i} u_t+\gamma u_{xx}+ |u|^{p-1}u=0,
\end{equation}
where the coordinate is $\bm{x}=(x,t)$, the state $u$ is complex, and $\mathrm{i}$ denotes the imaginary unit.
This equation is a governing equation of the wave function in quantum mechanics.
The linear version, which does not have the third term, is hyperbolic.
We simply treated a complex state $u$ as a 2D real vector $(w,v)$ for $u=w+\mathrm{i}v$ for training
We used $\gamma=0.5$ and $p=3$, which resulted in the following equations:
\begin{equation}
  \begin{aligned}
    w_{t}  + 0.5v_{xx} + (w^2+v^2)v =0, \\
    v_{t}  - 0.5w_{xx} - (w^2+v^2)w =0.
  \end{aligned}
\end{equation}
We used the domain $\Omega=[-5,5]\times[0,\pi/2]\ni(x,t)$ and the periodic boundary condition.
The initial condition was $u(x,0)=u_0(x)=\mathrm{sech}(x)$.
For evaluation, the numerical solution $u$ at $N_e=256\times201$ uniformly spaced points was provided, at which we calculated the relative error $\mathcal L$ of the absolute value $h=\sqrt{u^2+v^2}$.

For CPINNs, we folded each coordinate by the periodization trick and learned the initial and boundary conditions.
To learn the initial condition, $N_0=50$ collocation points $x_0,x_1,\dots,x_{N_0-1}$ were randomly sampled, with which the mean squared error of the state was calculated; namely $\frac{1}{N_0}\sum_{j=0}^{N_0-1} |\tilde{u}(x_j,0)-u_0(x_j)|^2_2$.
Also for the boundary condition, $N_b=50$ collocation points $t_0,t_1,\dots,t_{N_b-1}$ were randomly sampled, with which the mean squared errors of the state and the derivative were calculated; namely $\frac{1}{N_b}\sum_{j=0}^{N_b-1} |\tilde{u}(-5,t_j)-\tilde{u}(5,t_j)|^2_2$ and $\frac{1}{N_b}\sum_{j=0}^{N_b-1} |\tilde{u}_x(-5,t_j)-\tilde{u}_x(5,t_j)|^2_2$.
Then, the neural network was trained to minimize the sum of these three errors and the physics-informed loss.

\begin{figure}[t]
  \centering
  \scriptsize
  \tabcolsep=0mm
  \includegraphics[scale=0.7]{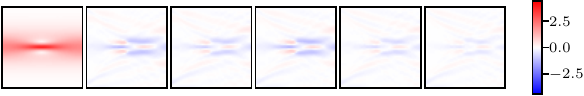}\hspace*{-3mm} \\[-1mm]
  NLS ($\times 10^2$)                                                 \\[1mm]
  \includegraphics[scale=0.7]{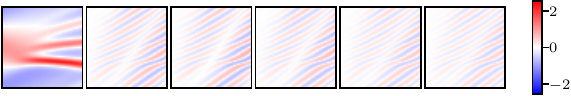}\hspace*{-3mm}           \\[-1mm]
  KdV ($\times 10^2$)                                                   \\[1mm]
  \includegraphics[scale=0.7]{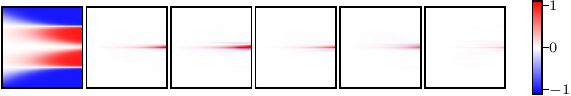}\hspace*{-3mm} \\[-1mm]
  AC ($\times 10^2$)                                         \\[1mm]
  \includegraphics[scale=0.7]{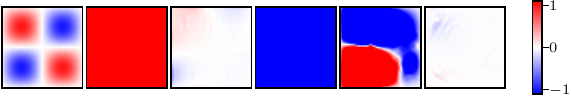}\hspace*{-3mm} \\[-1mm]
  Poisson with $s=2$  ($\times 10^3$) \hspace*{5mm}                \\[1mm]
  \includegraphics[scale=0.7]{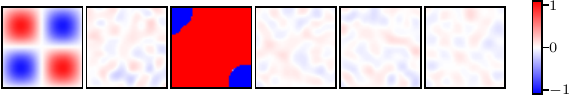}\hspace*{-3mm} \\[-1mm]
  Poisson with $s=4$  ($\times 10^3$) \hspace*{5mm}                \\
  \caption{Example results at competitive number $N$ of collocation points (on vertical green line in Fig.~\ref{fig:PINNsResults}).
    The leftmost panel shows the true solution.
    The remaining panels show the residuals of PINNs' results using uniformly random sampling, uniformly spaced sampling, LHS, Sobol sequence, and GLT, from left to right.
    The residuals are multiplied by the factors in parentheses.
  }
  \label{fig:PINNsExample}
\end{figure}

\paragraph{Korteweg--De Vries (KdV) Equation}
A 1D KdV equation is a hyperbolic PDE expressed as
\begin{equation}
  u_t+\lambda_1 uu_x+\lambda_2u_{xxx}=0,
\end{equation}
for the coordinate $\bm{x}=(x,t)$.
The KdV equation is a model of shallow water waves.
We used $\lambda_1=1$ and $\lambda_2=0.0025$, the domain $\Omega=[-1,1]\times[0,1]\ni(x,t)$, and the periodic boundary condition.
The initial condition was $u(x,0)=u_0(x)=\cos(\pi x)$.
For evaluation, the numerical solution $u$ at $N_e=512\times201$ uniformly spaced points were provided.
For the system identification, we treated the parameters $(\theta_1,\theta_2)=(\lambda_1,\lambda_2)$ to be learnable and initialized them to zero.

\paragraph{Allen--Cahn Equation}
A 1D AC equation is a parabolic PDE expressed as
\begin{equation}
  u_t-qu_{xx}+p(u^3-u)=0,
\end{equation}
for the coordinate $\bm{x}=(x,t)$.
The AC equation describes the phase separation of co-polymer melts.
We used $q=0.0001$, $p=5$, the domain $\Omega=[-1,1]\times[0,1]\ni(x,t)$, and the periodic boundary condition.
The initial condition was $u(x,0)=u_0(x)=x^2\cos(\pi x)$.
For evaluation, the numerical solution $u$ at $N_e=512\times201$ uniformly spaced points was provided.
For the system identification, we treated the parameters $(\theta_1,\theta_2)=(q,p)$ to be learnable and initialized them to zero.

\paragraph{Poisson's Equation}
2D Poisson's equation is expressed as
\begin{equation}
  \Delta u+f=0,
\end{equation}
where $\Delta$ is Laplacian and $f$ is a function of coordinates, state, and its first-order derivatives.
It is elliptic if the term $f$ is linear.
The case of $s=2$ dimensions with the coordinate $\bm{x}=(x,y)$ is expressed as
\begin{equation}
  u_{xx}+u_{yy}+f(x,y,u_x,u_y)=0.
\end{equation}
For the coordinate $\bm{x}=(x_1,\dots,x_s)$ in general, we used the domain $\Omega=[0,1]^s$, the Dirichlet boundary condition $u(\bm{x})=0$ at $\partial\Omega$, and $f(\bm{x})=\prod_{k=1}^s\sin(k\pi x_k)$.
Then, we obtained the analytic solution as $u(\bm{x})=\prod_{k=1}^s\sin(k\pi x_k)/(-sk^2\pi^2)$; there are $k^s$ peaks or valleys.
We set $k=2$ in this paper.
To ensure the boundary condition, we multiplied the output $\tilde u$ of the neural network by $\prod_{k=1}^sx_k(x_k-1)$.
For evaluation, we obtained the analytical solution $u$ at collocation points determined by the uniform spaced sampling, whose number is $N=999^2$, $49^4$, $13^6$, and $7^8$ for $s=2$, $4$, $6$, and $8$-dimensional cases, respectively.

\paragraph{Burgers' Equation}
1D Burgers' equation is expressed as
\begin{equation*}
  u_t+uu_x-\nu u_{xx}=0
\end{equation*}
for the coordinate $\bm{x}=(x,t)$.
Burgers' equation is a convection-diffusion equation describing a nonlinear wave.
We used $\nu=(0.01/\pi)$, the domain $\Omega=[-1,1]\times[0,1]\ni(x,t)$, the Dirichlet boundary condition $u(-1,t)=t(1,t)=0$, and the initial condition $u(x,0)=u_0(x)=-\sin(\pi x)$.
For evaluation, the numerical solution $u$ at $N_e=256\times 100$ uniformly spaced points was provided.
From these points, we used $N_0=256$ and $N_b=100$ collocation points for learning the initial condition and boundary condition, respectively.
The loss functions for these conditions were identical to those for the NLS equation.

\section{Additional Results}
\subsection{Example Results}
In this section, we show the true solutions and the residuals of example results with such $N$ in Fig.~\ref{fig:PINNsExample}.
We can see that the proposed GLT yielded the smallest errors.

\begin{figure}[t]
  \tabcolsep=1mm
  \centering
  \scriptsize
  \begin{tabular}{cc}
    \includegraphics[scale=0.9]{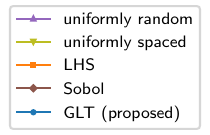}                     &
    \hspace*{-2mm}\includegraphics[scale=0.8]{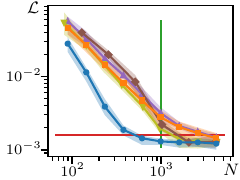}   \\[-1.0mm]
                                                                           &
    \hspace*{4mm} NLS                                                        \\[0.5mm]
    \hspace*{-2mm}\includegraphics[scale=0.8]{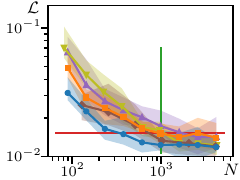} &
    \hspace*{-2mm}\includegraphics[scale=0.8]{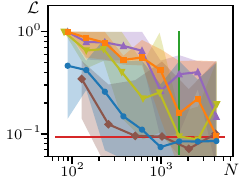}    \\[-1.0mm]
    \hspace*{4mm} KdV                                                      &
    \hspace*{4mm} AC                                                         \\[0.5mm]
  \end{tabular}
  \caption{The number $N$ of collocation points and the relative error $\mathcal L$.}
  \label{fig:PINNsResultsLearning}
\end{figure}

\begin{table}[t]
  \centering
  \scriptsize
  \caption{Trade-Off between Number $N$ of Collocation Points and Relative Error $\mathcal L$.}
  \label{tab:datasets_learning}
  \small
  \begin{tabular}{c@{ }lrrr}
    \toprule
                                                                                     &                      & \multicolumn{3}{c}{relative error $\mathcal L^\ddagger$}                                                    \\
    \cmidrule(lr){3-5}
                                                                                     &                      & \multicolumn{1}{c}{NLS}                                  & \multicolumn{1}{c}{KdV} & \multicolumn{1}{c}{AC} \\
    \midrule
    \scalebox{1.0}[0.8]{\textcolor{C4}{$\blacktriangle$}}                            & uniformly random     & 3.18                                                     & 17.30                   & 382.51                 \\
    \rotatebox[origin=c]{180}{\scalebox{1.0}[0.8]{\textcolor{C8}{$\blacktriangle$}}} & uniformly spaced     & 1.98                                                     & 16.08                   & 94.33                  \\
    \scalebox{0.7}{\textcolor{C1}{$\blacksquare$}}                                   & LHS                  & 2.78                                                     & 15.14                   & 158.71                 \\
    \rotatebox[origin=c]{45}{\scalebox{0.6}{\textcolor{C5}{$\blacksquare$}}}         & Sobol                & 2.21                                                     & 13.28                   & 94.35                  \\
    \textcolor{C0}{$\bullet$}                                                        & GLT & \textbf{1.31}                                            & \textbf{12.30}          & \textbf{84.50}         \\
    \midrule
    \textcolor{C0}{$\bullet$}                                                        & GLT (with tricks)    & \textbf{1.22}                                            & \textbf{2.19}           & \textbf{0.93}          \\
    \bottomrule                                                                                                                                                                                                           \\[-3mm]
  \end{tabular}\\
  \raggedright
  $\ddagger$ relative error $\mathcal L$ at competitive \# of points $N$ (on vertical green line in Fig.~\ref{fig:PINNsResultsLearning}). Shown in the scale of $10^{-3}$.\\
\end{table}

\subsection{Without Periodization Tricks}
We have proposed periodization tricks that ensure the periodicity assumption of the integrand and the initial and boundary conditions.
In this section, we examined the scenarios where these tricks were not applied.
We presented the results in Fig.~\ref{fig:PINNsResultsLearning} and Table~\ref{tab:datasets_learning} in the same way as Fig.~\ref{fig:PINNsResultsLearning} and Table~\ref{tab:datasets_learning}.
We can observe that the proposed GLT surpasses all other sampling methods, with a few exceptions at certain values of $N$ in the AC equation.

Most importantly, the relative error $\mathcal L$ significantly deteriorated for all sampling methods when these tricks were not applied.
Ensuring the initial and boundary conditions that include periodicity, rather than learning them, is crucial to achieving high performance, and is also effective in all sampling methods.
On the other hand, while the proposed GLT assumes the periodicity of the loss function in theory, it is not necessarily essential in practice.

\begin{figure}[t]
  \tabcolsep=1mm
  \centering
  \footnotesize
  \hspace*{-2mm}\includegraphics[scale=0.8]{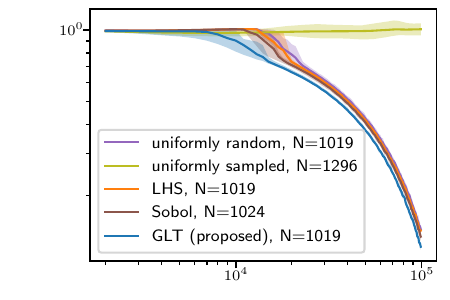}   \\[-3.0mm]
  \hspace*{4mm} Poisson's equation ($s=4$)                                     \\[0.5mm]
  \hspace*{-2mm}\includegraphics[scale=0.8]{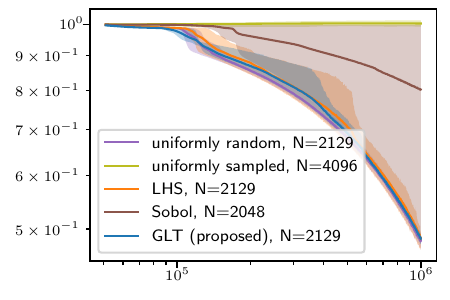}   \\[-3.0mm]
  \hspace*{4mm} Poisson's equation ($s=6$)                                     \\[0.5mm]
  \hspace*{-2mm}\includegraphics[scale=0.8]{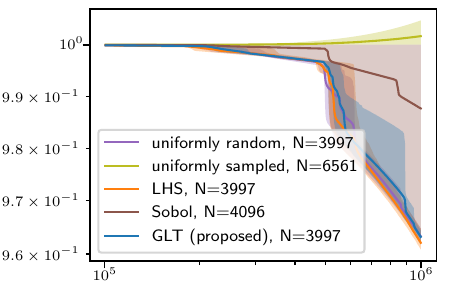}   \\[-3.0mm]
  \hspace*{4mm} Poisson's equation ($s=8$)                                     \\[0.5mm]
  \caption{ The number of iterations and the relative error $\mathcal L$.}
  \label{fig:HigherDim}
\end{figure}

\subsection{Higher Dimensional Case}
We assessed the proposed GLT and other comparison methods using a higher-dimensional Poisson's equation.
Our findings indicate that for dimensions $s\ge 6$, the solution did not converge within a reasonable time.
Therefore, we employed a similar number of collocation points, set the learning rate at $10^{-4}$, and summarized the resultant relative errors $\mathcal L$ in relation to the number of iterations in Fig.~\ref{fig:HigherDim}.
The solid lines represent the average relative error $\mathcal L$ across five trials, while the shaded areas show the range of maximum and minimum errors.
The solid lines denote the relative error $\mathcal L$ averaged over five trials, and the shaded areas denote the maximum and minimum errors.

The uniformly spaced sampling was ineffective due to its discretization error of $O(1/N^{\frac{1}{s}})$, which becomes more significant in higher dimensions.
At $s=4$, the proposed GLT most effectively reduced the relative error $\mathcal L$, followed by the Sobol sequence.
However, when the dimension number was increased to $s=6$ or $s=8$, the uniformly random sampling (i.e., the Monte Carlo method) proved to be as effective as, or superior to, the proposed GLT method.
This is because the Monte Carlo method produces the discretization error of $O(1/N^{\frac{1}{2}})$, which remains constant regardless of the dimension number $s$, in contrast to the proposed GLT's discretization error of $O(\frac{(\log N)^{\alpha s}}{N^\alpha})$, which is sensitive to the dimension number $s$.

Despite these variations, the performance difference was minimal, and the proposed GLT still demonstrated some effectiveness in higher dimensions.
Conversely, approximately half of the Sobol sequence trials showed no progress.
While the exact reasons are unclear, we observed that the Sobol sequence sometimes failed to approximate the integrated loss accurately in the early stages of training, preventing further progress.

In conclusion, the proposed GLT is the preferred option for simulations in four or fewer dimensions, as commonly required in most physical simulations.
While not always the optimal choice in higher dimensions, it maintains a certain level of effectiveness.
This represents an advantage over other quasi-Monte Carlo methods, such as the Sobol sequence.
Moreover, PINNs struggled to efficiently learn high-dimensional PDEs of six or more dimensions with realistic time and precision.
Therefore, beyond selecting collocation points, further innovation is necessary.

\end{document}